\documentclass[runningheads]{llncs}
\usepackage{graphicx}
\usepackage{times}
\usepackage{marvosym}

\usepackage{graphicx} 
\usepackage{bm}
\usepackage[numbers]{natbib}
\usepackage{xcolor}
\usepackage{amsfonts}
\usepackage{multirow}
\usepackage[ruled]{algorithm2e}
\usepackage{amsmath}
\usepackage{hyperref}
\usepackage{subcaption}

\begin{document}
\title{Generative Model for Heterogeneous Inference}
%
%
\author{Honggang Zhou \inst{1} \and
Yunchun Li \inst{1}\and
Hailong Yang \inst{1}\textsuperscript{(\Letter)} \and
Wei Li \inst{1}\and
Jie Jia \inst{2}
}
\authorrunning{H. Zhou et al.}
%
\institute{School of Computer Science and Engineering, Beihang University, Beijing, China \email{\{zhg,lych,hailong.yang,liw\}@buaa.edu.cn}\\
\and Taiyuan University of Technology, Taiyuan, China \email{jiajie1921@link.tyut.edu.cn}
}
\maketitle              
\begin{abstract}
Generative models (GMs) such as Generative Adversary Network (GAN) and Variational Auto-Encoder (VAE) have thrived these years and achieved high quality results in generating new samples. Especially in Computer Vision, GMs have been used in image inpainting, denoising and completion, which can be treated as the inference from observed pixels to corrupted pixels. However, images are hierarchically structured which are quite different from many real-world inference scenarios with non-hierarchical features. These inference scenarios contain heterogeneous stochastic variables and irregular mutual dependences. Traditionally they are modeled by Bayesian Network (BN). However, the learning and inference of BN model are NP-hard thus the number of stochastic variables in BN is highly constrained. In this paper, we adapt typical GMs to enable heterogeneous learning and inference in polynomial time. We also propose an extended autoregressive (EAR) model and an EAR with adversary loss (EARA) model and give theoretical results on their effectiveness. Experiments on several BN datasets show that our proposed EAR model achieves the best performance in most cases compared to other GMs. Except for black box analysis, we've also done a serial of experiments on Markov border inference of GMs for white box analysis and give theoretical results.
\end{abstract}
\section{Introduction}
\label{section_introduction}
Learning from heterogeneous distribution and performing effective inference are essential for many applications. For example finding interactions between genes \cite{Friedman2000}, fault diagnosing \cite{Cai2014}, risk analysis \cite{Trucco2008}, etc. Traditional approach to cope with these problems is BN, which factors joint probability distribution into conditional probability distribution as follows:
 \begin{equation}
 p(\bm{v})=\prod_{i=1}^Dp(v_i|\bm{v}_{parents(i)}))
 \end{equation}
These stochastic variables $v_i$ are arranged into a directed acyclic graph and $v_{parents(i)}$ is the parent of $v_i$. BN can be constructed using expert knowledge or learning from data sampled from joint probability distribution by structure learning and parameter estimation. Unfortunately, the structure learning \cite{Chickering1996} and inference (both approximate and exact inference) \cite{Cooper1990, Dagum1993} are all NP-Hard problems. However, the function of BN can be achieved without explicitly constructing a DAG. Any GM that approximates posterior probability distribution conditioned on any observations can be treated as a BN in black box.

Recent development of GMs such as autoregressive distribution estimator \cite{larochelle2011}, Variational Auto-Encoders (VAE) \cite{kingma2013}, Generative Adversarial Network (GAN) \cite{goodfellow2014} and pixel recurrent neural networks \cite{oord2016} leverages neural networks to achieve higher speed and accuracy than traditional graphical models which require heavy sampling. Although these models have been successfully used in image inpainting \cite{Yeh2016,Pathak2016,Ishikawa2017}, and denoising \cite{Zhang2016} which can be seen as the inference of pixels, little work has been done to adapt these models for heterogeneous inference. The prevalent method for pixel inference usually compresses the corrupted image into latent representation using Convolutional Neural Network (CNN) and tries to restore the original image using Deconvolutional Neural Network. Images are highly hierarchical and repetitive. Low-level pixels form textures and lines. Textures and lines then form shapes and objects. In this way, the whole image can be compressed into low entropy latent representation. Due to such hierarchy and repetition, high-level latent representation is not twisted too much by the corrupted part. Therefore, the corrupted part can then be inferred from latent representation. Unfortunately, heterogeneous inference does not possess the hierarchical features. Different stochastic variables are heterogeneous and those variables do not have high-level low entropy representation. The network structure used by image inference is not appropriate for heterogeneous inference. We need new models for heterogeneous inference with non-hierarchical features.

GMs have been prevalently used to model a random vector $\bm{x}$ and a latent vector $\bm{z}$. Inference can be made between $\bm{x}$ and $\bm{z}$. The task of heterogeneous inference can be seen as inference of latent vector $\bm{z}$ given observation vector $\bm{o}$. However, $\bm{z}$ and $\bm{o}$ are not fix-sized vector and thus traditional GM models cannot be directly used for heterogeneous inference tasks. We make adaption using the following method. We use a fix-sized one-hot vector to encode both $\bm{o}$ and $\bm{z}$. The value of unobserved variables in $\bm{o}$ is simply set to zero. $\bm{z}$ contains both observed variables and unobserved variables but we only use the value of unobserved variables. Note that we only consider categorical distribution. Then we replace vectors of original GMs with $\bm{o}$ and $\bm{z}$. The detail is described in Section \ref{section_label}.


\begin{figure}[ht]
\vskip 0.2in
\begin{center}
\centerline{\includegraphics[width=0.5\columnwidth]{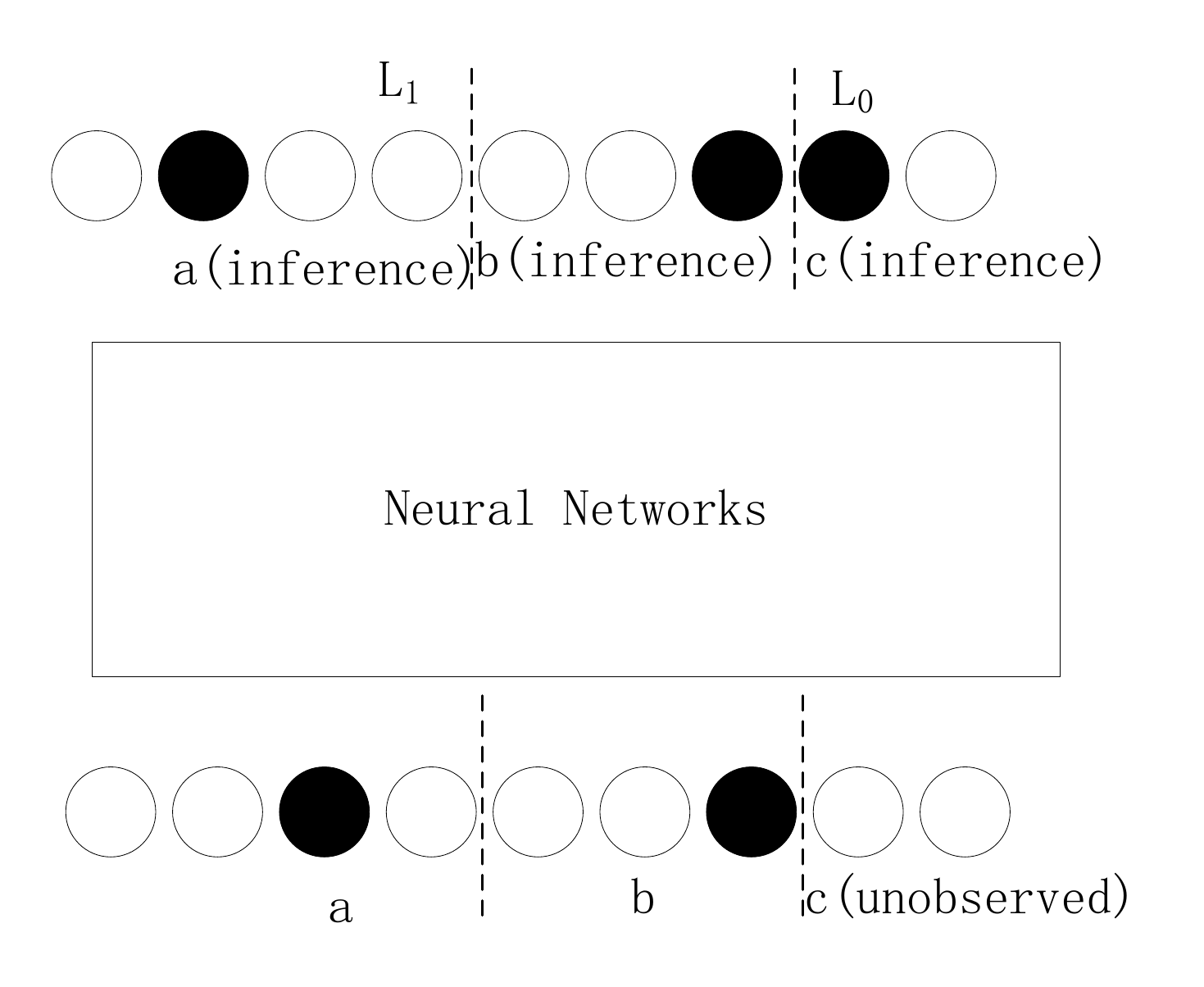}}
\caption{Overview of our proposed EAR model.}
\label{figure_overview}
\end{center}
\vskip -0.2in
\end{figure}

Except for adapting existing models, we also propose an EAR model and an EARA model that adds an adversary loss to EAR, as shown in Figure \ref{figure_overview}. EAR directly regresses latent vector from observation vector through a neural network. Our proposed method is based on the AR model and drops the order constraint of AR because the optimal order of variables is hard to determine. As proved in Section {section.ear.proof}, we can build a redundant network directly approximating the distribution of every stochastic variable conditioned on other stochastic variables. The structure is similar to autoencoders. However, they are different in several aspects. Firstly, the purpose of traditional autoencoders is to obtain low dimension representation, whereas the purpose of EAR is to obtain the output vector. Secondly, autoencoders usually have only one hidden layer, whereas the EAR model has no constraints on the width or depth of hidden layers. Lastly, EAR does not need hidden representation and thus all kinds of structures can be used.

\section{The Model}
\label{section_label}
\subsection{Background and Notations}
\label{section_notation}
Any machine learning problems can be seen as an inference problem. Discriminative models infer the conditional probability distribution of label $y$ given input $x$. Generative models usually model the manifold of data space and inference samples near the modeled manifold. However, the inference of BN is different. BN infers unobserved variables given any observed variables (or even no variables at all). The observed variables can be chosen arbitrarily from all variables. The information we can get from BN is very rich. We can answer any question and predict the behavior of any variables under any observations. However, such property is a double-edged sword. The learning and inference of BN are all NP-hard problems. Many real-world examples have tens of thousands of variables where BN cannot be applied. We thus present a neural network based solution for such kind of inference which breaks the NP-hard constraints by compromising on the accuracy.

To make the statement clear, we assume that we are given N training examples {$X_1$,...,$X_N$}. Each training example is composed of M stochastic variables {$X^1$,...,$X^M$}. Except for $\bm{o}$ and $\bm{z}$ in Section \ref{section_introduction}, we also use another set of notation for convenience. We use $X$ for a sample arranged in onehot format and $\hat{X}$ for the prediction of models. Although the stochastic variables in this paper is discrete, we use linear activation function rather than softmax in output layer as \cite{redmon2016} do. The reason is that squared error loss with linear activation in output layer is faster in our problem.
\subsection{Extended Autoregressive Model}
\label{section.ear.proof}
As shown in Figure \ref{figure_overview}, EAR model can be seen as a cluster of classifiers with parameter share except for the weights of the last layer. Note that we take unobserved state as an additional state for each variable rather than modeling a additional variable for observation state like \cite{Yichuan2012}. Each classifier is a discriminative model. When combined together, they equal a generative model. We have the following theorem
\begin{theorem}
A cluster of discriminative classifier $\bigcup_i p_{\theta_i}(X_i|\bm{o})$ for all variables in $X$ equals a generative model $p(X)$.
\end{theorem}
\begin{proof}
The equivalent generative model can be factorized to
\begin{equation}
\begin{aligned}
P(X_1,...,X_M)=&\prod P(X_1)P(X_2|X_1)P(X_3|X_1, X_2)...\\
&P(X_M|X_1,...,X_{M-1})
\end{aligned}
\end{equation}
with each factor representing for a discriminative classifier.
\end{proof}
The sharing parameter structure of EAR can achieve the best performance any independent classifiers can achieve. Before going any further, we first introduce a lemma first introduced in \cite{cybenko1988} about the capacity of neural networks
\begin{lemma}
\label{lemmaT}
Any function can be approximated by a three layer continuous valued neural networks $T_{W_1,W_2}$.
\end{lemma}

Formally, suppose we have M independent classifiers $\bigcup_i p_{\theta_i}^*(\bm{X}_i|\bm{X})$ which achieve best performance for given metric $\epsilon$ by metric function $f$ training on dataset $\mathcal{D}$. Our EAR model is denoted as $EAR_{\theta_{EAR}}(X|\bm{o})$. Then we have the following theorem
\begin{theorem}
$\exists EAR^*_{\theta_{EAR^*}}(X|\bm{o}), \forall i$, we have \\
 $f(EAR^*_{\theta_{EAR^*}}(X_i|\bm{o}))=f(p^*_{\theta^*_{i}}(X_i|\bm{o}))$
\end{theorem}

\begin{proof}
By applying Lemma \ref{lemmaT}, $\forall i, \exists T_{W_1^i,W_2^i}$, \\
making $f(T_{W_1^i,W_2^i})=f(p^*_{\theta^*_{i}}(X_i|X))$.
We have
\begin{equation}
\nonumber
\begin{aligned}
W_1^{EAR}=&[W_1^1,...,W_1^M], \\
W_2^{EAR}=&
\begin{bmatrix}
W_2^1&&0\\
&\ddots&\\
0&&W_2^M
\end{bmatrix}
\end{aligned}
\end{equation}
making
\begin{equation}
\nonumber
\begin{aligned}
\forall i,f(EAR_{[W_1^{EAR},W_2^{EAR}]}(X_i|X))=f(p^*_{\theta^*_{i}}(X_i|X))
\end{aligned}
\end{equation}
\end{proof}
The loss function of EAR can be written in three terms:
\begin{equation}
\label{eq.ear}
\begin{aligned}
\mathcal{L}&=\mathcal{L}_0+\mathcal{L}_1+\mathcal{L}_2 \\
\mathcal{L}_{0}&=\alpha||Mask\odot(X-\hat{X})||_2^2 \\
\mathcal{L}_{1}&=\beta||(1-Mask)\odot(X-\hat{X})||_2^2 \\
\mathcal{L}_{2}&=\gamma||\theta||_1  \\
\end{aligned}
\end{equation}
where $\hat{X}$ is the output of EAR. The first term $\mathcal{L}_{0}$ represents for latent variable prediction error. The second term $\mathcal{L}_{2}$ represents for stability regularizer which forces the network to predict what's observed. The last item $\mathcal{L}_{3}$ represents for parameter regularizer. Here we use L1 normalizer because BN is usually sparse thus we want the network to be sparse too. The $Mask$ is 1 for latent variables and 0 for observed variables.

\subsection{Extended Autoregressive Model with Adversary Loss}
We also present another model which adds an adversary loss on EAR called EARA. The EAR model is trained to produce $\hat{X}$ trying to fool discriminator $D$ and the discriminator $D$ tries to discriminate ground truth $X$ from generated distribution $\hat{X}$. The difference from GAN is that when optimizing EAR model, we add EAR loss together with generator loss. The training procedure is shown in Algorithm \ref{algoeara}.

\begin{algorithm}
\caption{The training procedure of EARA}
\DontPrintSemicolon
\For{number of training iterations}{
Sample minibatch of samples $\bm{X^1},...,\bm{X^n}$.\;
Generate minibatch of observation vectors $\bm{o^1},...,\bm{o^n}$.\;
\tcp{Update discriminator:}
\begin{equation}
\nonumber
\begin{aligned}
\theta_D+=&\bigtriangledown_{\theta_{D}}\frac1n\sum_{i=1}^n[logD(\bm{X^i})+\\
&log(1-D(EAR(o^i)))].
\end{aligned}
\end{equation}\;
Sample minibatch of samples $\bm{X^1},...,\bm{X^n}$.\;
Generate minibatch of observation vectors $\bm{o^1},...,\bm{o^n}$.\;
\tcp{Update EAR:}
\begin{equation}
\nonumber
\begin{aligned}
\theta_{EAR}+=&\bigtriangledown_{\theta_{EAR}}\frac1n\sum_{i=1}^n[log(1-D(EAR(o^i)))+\\
&\mathcal{L}].
\end{aligned}
\end{equation}\;
\tcp{$\mathcal{L}$ is the same as Equation \ref{eq.ear}}
}
\label{algoeara}
\end{algorithm}
\subsection{Adaption of Generative Models}
\subsubsection{Restricted Boltzmann Machine}
RBM \cite{freund1992} is an energy based two-layer generative model including visible layer $\bm{v}$ and hidden layer $\bm{h}$. The energy of RBM can be written as:
\begin{equation}
E(\bm{h,v})=-(\bm{a^\top v}+\bm{b^\top h}+\bm{v^\top Wh})
\end{equation}
where $\bm{a}$, $\bm{b}$ and $\bm{W}$ are parameters. From energy function, joint probability distribution of $\bm{v}$ and $\bm{h}$ can be written as:
\begin{equation}
P(\bm{h},\bm{v})=\frac{exp(-E(\bm{h},\bm{v}))}{\sum_{\bm{v},\bm{h}}exp(-E(\bm{h},\bm{v}))}
\end{equation}
The conditional probability distribution is
 \begin{equation}
 \label{RBM_condition_prob}
 \begin{aligned}
P(v_i|\bm{h}) = \sigma(a_i+\bm{W}_i\bm{h})\\
P(h_j|\bm{v}) = \sigma(b_j+\bm{W}_j\bm{v})
\end{aligned}
 \end{equation}
where $\sigma(x)=1/(1+e^{-x})$ is the sigmoid function.
Our RBM is trained using Contrastive Divergence (CD) algorithm \cite{Carreira2005}. The gradient of parameter is
\begin{equation}
\begin{aligned}
\Delta \bm{W}=&\epsilon\frac{\partial logP(\bm{v})}{\partial \bm{W}} \\
=&\epsilon (<\bm{vh}>_{data}-<\bm{vh}>_{model})\\
\Delta \bm{a}=&\epsilon (<\bm{v}>_{data}-<\bm{v}>_{model})\\
\Delta \bm{a}=&\epsilon (<\bm{h}>_{data}-<\bm{h}>_{model})\\
\end{aligned}
\end{equation}
Where $\epsilon$ is the learning rate. $<\cdot>_{model}$ is the expectation under the distribution defined by model. $<\cdot>_{model}$ is the expectation under the distribution of data. The $\bm{h}$ in $<\bm{vh}>_{data}$ is obtained by sampling from $P(\bm{h}|\bm{v})$.
We adapt RBM for heterogeneous inference by replacing the input $\bm{v}$ with our observation vector $\bm{o}$. And replacing $\bm{v}$ in $<\bm{vh}>_{data}$ with $X$ when updating parameters in CD algorithm. During testing, $\hat{X}$ is derived using Equation \ref{RBM_condition_prob} and $\bm{h}$ is sampled from $P(\bm{h}|\bm{o})$.
\subsubsection{Wasserstein Generative Adversarial Networks}
Wasserstein Generative Adversarial Networks (WGAN)\cite{arjovsky2017} is a more stable version of GAN. WGAN trains a generator network $G(z;\theta_G)$ to produce samples from data distribution $p_{data}(x)$ from noise vector $z$ and a discriminator network $D(X;w_D)$ to discriminate samples produced by generator from real data. The generator and discriminator plays a Min-Max game in turn. After equilibrate the generator will generate samples to distribution $p_{data}(x)$.
The optimization function we use for heterogeneous inference is
\begin{equation}
\begin{aligned}
\mathop{max}\limits_{D}\mathop{min}\limits_{G}\mathbb{E}_{X\sim\mathbb{P}_{X}}[D(X;w_D)]-\\
\mathbb{E}_{\bm{o}\sim\mathbb{P}_{(\bm{o}|X)}}[D(G(\bm{o};\theta_G);w_D)]
\end{aligned}
\end{equation}
where $w_D$ and $\theta_G$ are network parameters, $\mathbb{P}_{(\bm{o}|X)}$ is the generating process of observation vector $\bm{o}$. The training procedure is the same as standard WGAN.

\subsubsection{Conditional Generative Adversarial Networks}
Conditional Generative Adversarial Networks (CGAN) \cite{mirza2014} an extension to standard GAN by conditioning generator and discriminator with label $\bm{y}$. We adapt CGAN by replacing label $\bm{y}$ with observation vector $\bm{o}$. Then the optimization function is:
\begin{equation}
\begin{aligned}
\mathop{max}\limits_{D}\mathop{min}\limits_{G}\mathbb{E}_{X\sim\mathbb{P}_{\bm{X}}, \bm{o}\sim\mathbb{P}_{\bm{o}|\bm{X}}}[D(\bm{X}|\bm{o} ;w_D)]-\\
\mathbb{E}_{\bm{z}\sim\mathbb{P}_{\bm{z}}, \bm{o}\sim\mathbb{P}_{(\bm{o}|\bm{X})}}[D(G(\bm{z}|\bm{o};\theta_G);w_D)]
\end{aligned}
\end{equation}

\subsubsection{Variational Autoencoder}
VAE \cite{kingma2013} use stochastic variational inference to deal with intractable posterior distribution by maximizing variational lower bound
\begin{equation}
\mathcal{L}=-D_{KL}(q(\bm{z}|\bm{x})||p(\bm{z}))+\mathbb{E}_{q(\bm{z}|\bm{x})}[-logq(\bm{x}|\bm{z})]
\end{equation}
We omit parameters for conciseness. Like autoencoders, VAE tries to restore the input of encoder at the output of decoder as the second term of $\mathcal{L}$ shows. When input is corrupted, the lower bound becomes tighter \cite{im2017}
\begin{equation}
\begin{aligned}
\mathcal{L}_{denoising}&=\mathbb{E}_q(\bm{z}|\bm{x})[logp(\bm{x},\bm{z})-logq(\bm{z}|\tilde{\bm{x}})]
\end{aligned}
\end{equation}

We adapt VAE for heterogeneous inference is by leveraging denoising lower bound $\mathbb{L}_{denoising}$ and replacing $\tilde{\bm{x}}$ with observation vector $\bm{o}$ and replacing $\bm{x}$ with sample vector $X$.
\subsubsection{Conditional Variational Auto-encoder}
Similar to CGAN, Conditional Variational Auto-encoder (CVAE) \cite{newkingma2014} makes encoder and decoder conditioned on label $\bm{y}$. We make the same adaption as in CGAN for heterogeneous inference by replacing label $\bm{y}$ with observation vector $\bm{o}$. The optimization lower bound is
\begin{equation}
\begin{aligned}
\mathcal{L}=&\mathbb{E}_{q(\bm{z}|X,\bm{o}}[logp(X|\bm{o},\bm{z})+logp(\bm{o})+\\
&logp(\bm{z})-logq(\bm{z}|X,\bm{o})]
\end{aligned}
\end{equation}
\section{Related Work}
\cite{freund1992} introduces RBM as a bipartite undirected graphical model with one layer of visible units $v$ and one layer of hidden units $h$. RBM models the joint probability distribution of $v$ and $h$ by assigning an energy function: $E(v,h)=-a^{\top} v-b^{\top} h+v^{\top} Wh$. Then $P(v,h)$, $P(v|h)$, $P(h|v)$ can be derived from $E(v,h)$. The original training procedure of RBM need expensive Gibbs sampling procedure. Later \cite{Carreira2005} introduced CD algorithm which speeds up RBM greatly by reducing the sampling steps of $E_{P(v,h)}(vh)$.
\cite{salakhutdinov2007} applies RBM to collaborative filtering which can be seen as a denoising RBM since the model assumes part of input is loss. The inference is performed by calculating latent representation $h$ and then calculating the distribution of lost input using $h$. \cite{Yichuan2012} propose a robust RBM structure which leverage two RBMs with one model the missing state of variables.

\cite{kingma2013} introduces VAE to enable inference and learning in directed probabilistic models, in the presence of continuous latent variables with intractable posterior distribution. By adopting reparameterization trick and Stochastic Gradient Variational Bayes (SGVB) estimator, VAE can be trained using stochastic gradient descent (SGD) thus scale to large datasets. \cite{newkingma2014} proposes CVAE by adding condition $\bm{y}$ to input layer and stochastic hidden layer. Thus CVAE can generate samples conditioned on $y$. The variety is useful in semi-supervised learning where only part of data has been labeled. \cite{im2017} proposes denoising VAE by training VAE to reconstruct clean inputs with noise injected at the input level and proposes a denoising variational lower bound to make training criterion tractable. The denoising VAE yields better average log-likelihood.

\cite{goodfellow2014} introduces GAN to learn the distribution of data with a generator and discriminator. The advantage of GAN to previous generative models is that no Markov chains or unrolled approximate inference networks are needed during training or generating new samples. Standard GAN uses determined transformation from random Gaussian variables to samples in data space. However, the input of generator is not necessarily whitened Gaussian noises. So we use observation vector as the input of generator in this paper. Similar to CVAE, there is also a variant of GAN called CGAN \cite{mirza2014}. CGAN is constructed by feeding the condition $y$ to both generator and discriminator. CGAN can generate samples consist of provided condition (such as the label of wanted samples).


\section{Experiments}
\subsection{Experimental Setup}
Our experiments are conducted with several publicly available BN datasets. The training data are obtained by Gibbs sampling and the test data is obtained by randomly choosing observations and performing inference using \emph{Junction Tree Algorithm} \cite{Lauritzen1988}. The size of training data is 10000 for all dataset. The size of test data is 1000 or less (if we cannot generate 1000 different observations).
The network we use in EAR is a fully connected network with Relu activation in the intermediate layer and linear activation in the output layer. The size of hidden layers is 64, 128, 128, 64. One exception is that our synthesized datasets because they are extremely small. The size of hidden layers is 10, 10 for synthesized datasets. The generator of EARA uses the same network configuration as EAR. The discriminator of EARA consists of a three-layer neural network whose size of hidden layer is 128. The optimizer we use for EAR and EARA is momentum optimizer with momentum equal 0.9.

The RBM model we use consists of 36 binary hidden variable and the K for CD algorithm is 10. The encoder and decoder of our VAE/CVAE implementation both consist of a three-layer neural networks whose hidden layer size is 128. The generator and discriminator of WGAN/CGAN is the same as EARA. The optimizer for WGAN/CGAN is RMSProp. We pick 20\% of test data for validation. $\alpha$, $\beta$ and $\gamma$ of EAR and EARA are 1.0, 1.0, 0.005 respectively.
\subsubsection{Datasets}
\textbf{Alarm -} Alarm dataset \cite{Beinlich1989} is the abbreviation for A Logical Alarm Reduction Mechanism. It is a diagnostic application used to explore probabilistic inference in belief networks. It has 37 nodes, 46 arcs, 509 parameters. The average Markov blanket size is 3.51, average degree is 2.49, maximum in-degree is 4.

\textbf{Asia -} Asia dataset \cite{Lauritzen1990} (sometimes called Lung Cancer dataset) is a small BN dataset. It contains 8 nodes, 8 edges and 18 parameters. The size of average Markov blanket is 2.5, the size of average degree is 2.0 and maximum in-degree is 2.

\textbf{Child -} Child dataset \cite{spiegelhalter1992} is an expert system for disease diagnose. It contains 20 nodes, 25 edges and 230 parameters. The size of average Markov blanket is 3.0, the size of average degree is 1.25 and maximum in-degree is 2.

\textbf{Insurance -} Insurance dataset \cite{binder1997} contains 27 nodes, 52 edges and 984 parameters. The size of average Markov blanket is 5.19, the size of average degree is 3.85 and maximum in-degree is 3.

\textbf{Survey -} Survey dataset \cite{scutari2014} contains 6 nodes, 6 edges and 21 parameters. The size of average Markov blanket is 2.67, the size of average degree is 2.0 and maximum in-degree is 2.

\textbf{WIN95PTS -} WIN95PTS \cite{win95pts} is an expert system for windows fault diagnose. It contains 76 nodes, 112 edges and 574 parameters. The size of average Markov blanket is 5.92, the size of average degree is 2.95 and maximum in-degree is 7.

\subsubsection{Evaluating Metrics}
\textbf{Absolute Deviation -} Absolution Deviation (AD) indicates the average absolute deviation error between predicted distribution and ground truth distribution which can be written as follows:
\begin{equation}
AD=\frac1N\sum_i^N\frac1{M}\sum_j^{M}\frac{1}{K_j}\sum_k^{K_j}|(max\{min\{\hat{X}_{j,k}^i,1\},0\}-p(X_j^i=X_{j,k}^i))|
\end{equation}
where $p(X_j^i=X_{j,k}^i))$ means the probability of j-th variable of i-th sample takes value k.
AD is similar to K-L divergence but is more intuitional for humans.

\textbf{KL Divergence -} KL divergence measures the deviation of predicted distribution and ground truth distribution
\begin{equation}
\begin{aligned}
KL=&-\frac1N\sum_i^N\frac1M\sum_j^M\sum_k^{K_j}P(X_j^i=X_{j,k})\cdot\\
&log\frac{max\{min\{\hat{X}_{j,k}^i,1\},e^{-6}\}}{P(X_j^i=X_{j,k})}
\end{aligned}
\end{equation}
KL divergence is a common metric for probability distribution comparison.

\textbf{Classification Accuracy -} The purpose of inference is mostly used to determine the most probable value of latent variable. Classification Accuracy represents for the average accuracy of classification and can be written as
\begin{equation}
ACC=\frac1N\sum_i^N\frac1M\sum_i^M1_{argmax(\hat{X}_j^i)=argmax(X_j^i)}
\end{equation}

\subsection{Results}
\label{section.result}
\begin{table}[ht]
\centering
\caption{Results for heterogeneous inference using different generative models.}
\label{result}
\begin{tabular}{c|c|c|c|c|c|c|c|c|c|c|c|c|c|c|c|c|c|c}
\hline
\multirow{2}{*}{Model}&\multicolumn{3}{|c|}{Alarm}&\multicolumn{3}{|c|}{Asia}&\multicolumn{3}{|c|}{Child}&\multicolumn{3}{|c|}{Insurance}
&\multicolumn{3}{|c|}{Survey}&\multicolumn{3}{|c}{Win95pts}\\\cline{2-19}
&AD&KL&ACC&AD&KL&ACC&AD&KL&ACC&AD&KL&ACC&AD&KL&ACC&AD&KL&ACC\\\hline
RBM&0.13&0.23&0.85&0.30&0.43&0.63&0.13&0.18&0.74&0.18&0.38&0.61
&0.18&0.03&0.80&0.16&0.16&0.91
\\\hline
WGAN&0.23&0.08&0.84&0.27&1.57&0.76&0.17&0.04&0.76&0.23&1.49&0.59
&0.34&2.28&0.65&0.23&0.24&0.91
\\\hline
CGAN&0.24&1.81&0.69&0.37&1.05&0.78&0.35&3.16&0.44&0.41&2.57&0.45
&0.36&3.01&0.60&0.37&1.31&0.68
\\\hline
VAE&0.18&2.39&0.81&0.45&4.57&0.46&0.24&4.38&0.63&0.24&5.14&0.54
&0.24&3.14&0.89&0.13&1.13&0.91
\\\hline
CVAE&0.10&\textbf{0.06}&0.90&0.08&0.13&0.88&0.13&\textbf{0.03}&0.79&0.16&0.48&0.65
&0.03&\textbf{0.00}&0.93&0.27&0.25&0.92
\\\hline
EAR&\textbf{0.05}&0.08&\textbf{0.96}&\textbf{0.07}&\textbf{0.05}&\textbf{0.91}&\textbf{0.09}&0.13&\textbf{0.82}&\textbf{0.09}&\textbf{0.13}
&\textbf{0.82}&\textbf{0.02}&\textbf{0.00}&\textbf{0.98}&\textbf{0.07}&\textbf{0.14}&\textbf{0.93}
\\\hline
EARA&0.10&0.41&0.92&0.12&0.24&0.90&\textbf{0.09}&0.13&\textbf{0.82}&0.15&0.67&\textbf{0.82}&0.15&0.44&0.96&0.11&0.31&0.92
\\\hline
\end{tabular}
\end{table}

\begin{figure}[ht]
\vskip 0.2in
\begin{center}
\centerline{\includegraphics[scale=0.33]{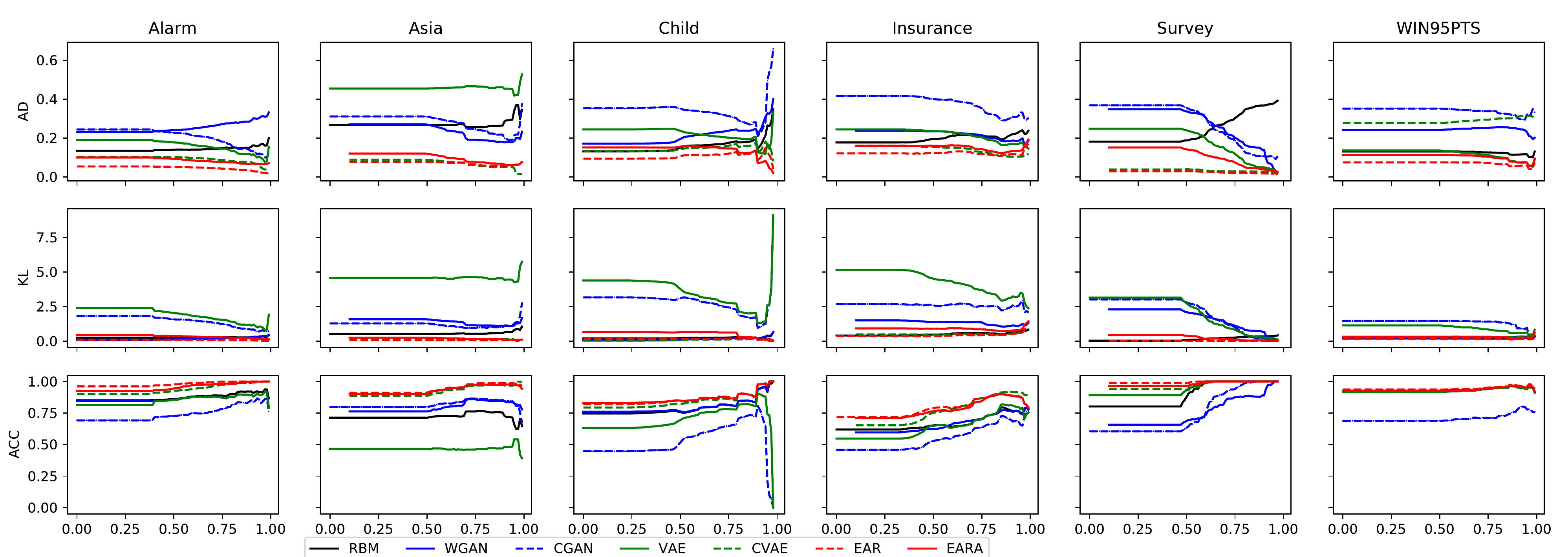}}
\caption{Comparison of models with different thresholds.}
\label{figure_threshold}
\end{center}
\vskip -0.2in
\end{figure}
The results of our experiment are shown in Table \ref{result}. In most cases, EAR has better accuracy than other models. The discriminator of WGAN can easily discriminate fake distribution from truth distribution and without reconstruction loss the generator struggles to learn the mapping between observation and posterior distribution. The adversary loss we add to EAR cause gradient conflict with EAR loss and have slightly inferior performance than EAR. We can also conclude that CVAE has similar performance to EAR. An obvious explanation is that decoder of CVAE has a very similar structure to EAR. RBM, WGAN, CGAN, VAE are inferior in almost all datasets.
It can also be seen that the size of BN datasets has little to do with inference accuracy of GMs which means these models can apply to more complex situations where traditional BN cannot be trained.

The posterior distribution might be hard to classify if $max\quad p(v_i=v_i^j)$ is very small. We have done another serial of experiments to demonstrate how accurate these models are with different thresholds. The results are shown in Figure \ref{figure_threshold}. With a higher threshold, these models tend to have higher ACC and lower AD, KL which means they make fewer mistakes when posterior distribution has less variance.

\subsection{Case Study}
Experiments in Section \ref{section.result} only consider a single metric for a system composed with all kinds of inference scenarios. Although it is a common practice for modern machine learning researches, such approach can not reveal the performance of basic inference steps. As a case study, we apply different GMs on Markov Border Inference to provide white box insights. In a directed graph model, the Markov border $\wp_v$ of a variable $v$ consists of its parent variables, child variables and the parent variables of child variables. Given $\wp_v$, $v$ is not affected by other variables. The inference from Markov border is the basic inference step for belief propagation. We will discuss the inference from Markov border in this section and given theoretical explanation for the results.

\subsubsection{Synthesized Datasets}
We build three synthesized datasets for three kinds of Markov borders. They are shown in Figure \ref{figure_markov_border}. All variables are binary and the conditional probability and prior probability are chosen randomly.
 The training procedure is similar to standard datasets. In the test procedure, we only care about the posterior probability of target variable when other variables are randomly observed. The results are shown in Table \ref{table_toy}. The inference from child variable to parent variable is a hard task for all models.

We offer an explanation for this phenomenon from the perspective of training data.
Assuming the prior probability of parent variable is $p_{parent} \sim \mathbb{B}(\alpha)$, where $\mathbb{B}$ is Bernoulli distribution and $\alpha \sim p_\alpha$. The conditional probability distribution of a child variable given the value of parent variable is

\begin{equation}
\begin{aligned}
&p_{condition}& \sim \mathbb{B} (\beta)  \quad if\quad parent=0\\
&& \sim \mathbb{B}(\gamma) \quad if\quad parent=1\\
&\beta \sim p_\beta\\
&\gamma \sim p_\gamma
\end{aligned}
\end{equation}
. Then the probability distribution of a child variable is $p_{child} \sim \mathbb{B}(\delta)$, with $\delta=\alpha\beta+\gamma-\alpha\gamma$. Under the hypothesis that hyper parameters are from the same distribution (Occam' razor principal), then $p_\alpha=p_\beta=p_\gamma=p_\pi$, the expectation of $\delta$  $\mathbb{E}_\delta=\mathbb{E}_\alpha\mathbb{E}_\beta+\mathbb{E}_\gamma-\mathbb{E}_\alpha\mathbb{E}_\gamma=\mathbb{E}_\pi$. But we have the variance $\sigma_\delta\ge\sigma_\alpha$. The variance of the distribution of child variable is larger than parent variable. So the training samples we generate according to Gibbs sampling are more skewed distributed which is a bad attribute for discriminative models. Some value of variable might occur too often or too rare making discriminative models hard to train.
\begin{figure}[ht]
\vskip 0.2in
\begin{center}
\caption{Three synthesized examples for Markov border inference.}
\label{figure_markov_border}
\begin{subfigure}[b]{0.3\textwidth}
\includegraphics[width=\linewidth]{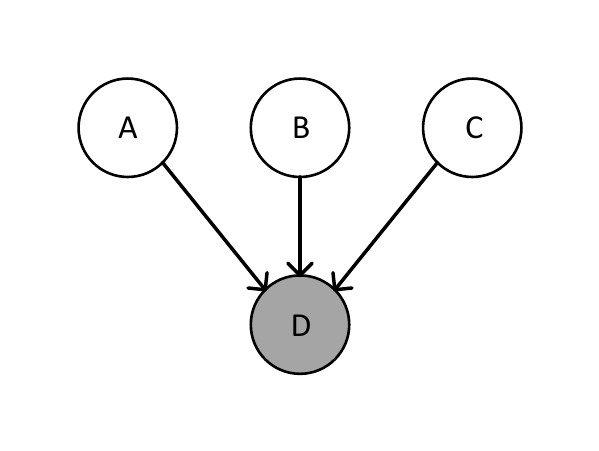}
\caption{(Dataset A) Inference from parent variable to child variable. }
\end{subfigure}
\begin{subfigure}[b]{0.3\textwidth}
\includegraphics[width=\linewidth]{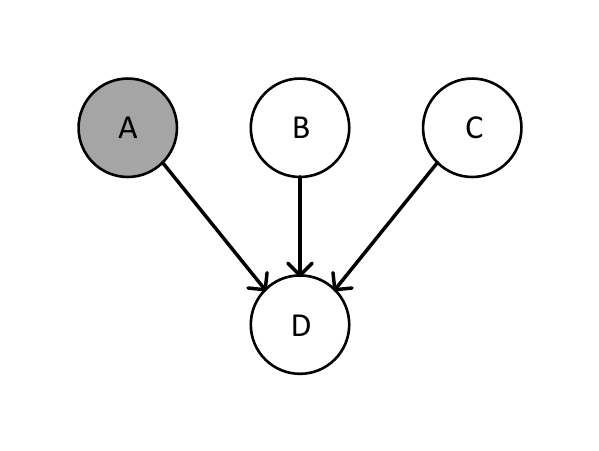}
\caption{(Dataset B) Inference from uncle variable to parent variable.}
\end{subfigure}
\begin{subfigure}[b]{0.3\textwidth}
\includegraphics[width=\linewidth]{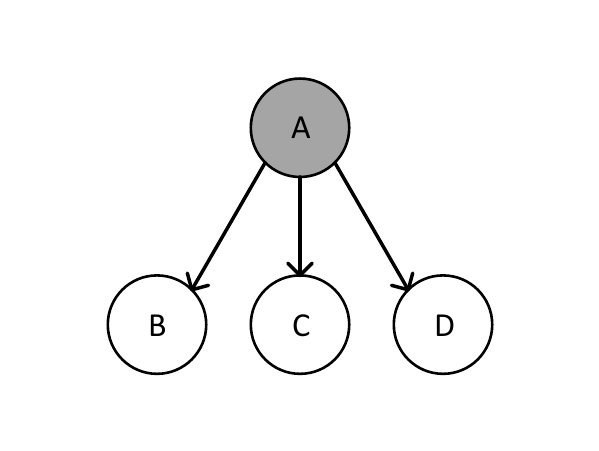}
\caption{(Dataset C) Inference from child variable to parent variable.}
\end{subfigure}
\end{center}
\vskip -0.2in
\end{figure}

\begin{table}[ht]
\caption{Results for Markov border inference with synthesized datasets.}
\label{table_toy}
\vskip 0.15in
\begin{center}
\begin{small}
\begin{sc}
\begin{tabular}{c|c|c|c|c|c|c|c|c|c}
\hline
\multirow{2}{*}{Model}&\multicolumn{3}{|c|}{A}&\multicolumn{3}{|c|}{B}&\multicolumn{3}{|c}{C}
\\\cline{2-10}
&AD&KL&ACC&AD&KL&ACC&AD&KL&ACC\\\hline
RBM&0.20&0.24&0.62&0.04&\textbf{0.03}&\textbf{1.00}&0.25&0.20&0.81
\\\hline
WGAN&0.22&0.17&0.62&0.22&0.14&\textbf{1.00}&0.17&0.11&0.81
\\\hline
CGAN&0.28&2.09&0.87&0.49&3.82&0.50&0.25&1.85&\textbf{0.87}
\\\hline
VAE&0.32&3.01&0.62&0.10&0.66&\textbf{1.00}&0.24&2.04&0.81
\\\hline
CVAE&0.13&0.12&\textbf{1.00}&0.06&0.04&\textbf{1.00}&0.22&0.16&0.75
\\\hline
EAR&\textbf{0.08}&\textbf{0.06}&0.93&\textbf{0.04}&0.04&\textbf{1.00}&\textbf{0.14}
&\textbf{0.03}&0.81
\\\hline
EARA&0.21&0.87&0.87&0.08&0.42&\textbf{1.00}&0.15&0.36&0.81
\\\hline
\end{tabular}
\end{sc}
\end{small}
\end{center}
\vskip -0.1in
\end{table}


\subsubsection{Comparison with Naive Classifier}
We build a twin classification network to EAR model in three synthesized datasets. We refer the classifier as Naive Classifier (NC). The only modification compared to EAR is that we add a mask when loss is calculated so that non-target variables don't propagate gradient. The results are shown in Table \ref{table_nc}. There is no obvious performance degradation of EAR compared with NC and the computation overhead added to NC is negligible both in space and time.
\begin{table}[ht]
\caption{Comparison on Markov border inference with synthesized datasets.}
\label{table_nc}
\vskip 0.15in
\begin{center}
\begin{small}
\begin{sc}
\begin{tabular}{c|c|c|c|c}
\hline
Dataset& Model& AD&KL&ACC\\
\hline
\multirow{2}{*}{A}
&EAR&\textbf{0.08}&\textbf{0.06}&\textbf{0.93}\\\cline{2-5}
&NC&0.10&0.11&\textbf{0.93}\\
\hline
\multirow{2}{*}{B}
&EAR&0.04&0.04&\textbf{1.00}\\\cline{2-5}
&NC&\textbf{0.03}&\textbf{0.03}&\textbf{1.00}\\
\hline
\multirow{2}{*}{C}
&EAR&0.14&\textbf{0.03}&\textbf{0.81}\\\cline{2-5}
&NC&\textbf{0.13}&0.07&\textbf{0.81}\\
\hline
\end{tabular}
\end{sc}
\end{small}
\end{center}
\vskip -0.1in
\end{table}

\section{Discussion}
This work is closely related to all kinds of denoising models in a way that the observation vector $\bm{o}$ in this paper is similar to the corrupted input in denoising models. However, there are several differences. Firstly, observation vector has a much higher entropy than a corrupted or occlusive input. Certain patterns are presumed in previous works \cite{Yichuan2012,im2017}, however there is no constraint for $\bm{o}$ in this work. An extreme case could be that there are no observed variables in the observation vector at all. Therefore, modeling the observation state is not a feasible option. Secondly, previous denoising models mostly use noises to enhance robustness \cite{vincent2008,Yichuan2012,im2017} of the model for unseen data. However, the inference from any kind of observed variable has not been explored by previous models. Lastly, previous denoising models rely on hidden representation. Our EAR model dose not rely on explicit hidden representation and thus has more flexibility.


\section{Conclusion}
In this work, we adapt typical generative models to perform heterogeneous inference and propose EAR, EARA models. Our EAR model achieves the best performance compared to existing models, which avoids the NP-hard constraint of traditional BN and can be learned from end to end. In the last, we present a case study by applying EAR and other GM models to Markov border inference with comprehensive analysis.
\section*{Acknowledgement}
This work is supported by National Key Research and Development Program of China (Grant No. 2016YFB1000304) and National Natural Science Foundation of China (Grant No. 61502019).
\bibliography{ref}
\bibliographystyle{plain}

\end{document}